\documentclass{article}
\PassOptionsToPackage{table}{xcolor}
\usepackage{iclr2027_conference,times}
\iclrfinalcopy
\usepackage[T1]{fontenc}
\usepackage[utf8]{inputenc}
\usepackage{microtype}
\usepackage{graphicx}
\usepackage{amsmath,amssymb,amsthm}
\usepackage{booktabs,tabularx}
\usepackage{multirow,pifont}
\usepackage{xcolor}
\usepackage{tikz}
\usetikzlibrary{arrows.meta,positioning,calc,fit,backgrounds}
\usepackage{hyperref}
\usepackage{url}
\newcommand{\method}{OmniMoE-VL}
\newcommand{\R}{\mathbb{R}}

\newcommand{\softmax}{\operatorname{softmax}}
\newcommand{\TopK}{\operatorname{TopKSoftmax}}
\newcommand{\Attn}{\operatorname{Attn}}
\newcommand{\LN}{\operatorname{LN}}
\newcommand{\rank}{\operatorname{rank}}
\newcommand{\diag}{\operatorname{diag}}
\newcommand{\KL}{\operatorname{KL}}
\newcommand{\neff}{n_{\mathrm{eff}}}
\newcommand{\ntgt}{n_{\mathrm{tgt}}}

\newtheorem{proposition}{Proposition}

\definecolor{ink}{HTML}{243547}
\definecolor{routeblue}{HTML}{365E95}
\definecolor{visionteal}{HTML}{347D78}
\definecolor{querypurple}{HTML}{796396}
\definecolor{routefill}{HTML}{EDF2F8}
\definecolor{visionfill}{HTML}{EDF5F2}
\definecolor{queryfill}{HTML}{F3EFF7}
\definecolor{softgray}{HTML}{F4F5F6}
\definecolor{headblue}{RGB}{214,234,248}
\definecolor{groupblue}{RGB}{238,246,255}
\definecolor{rowblueA}{RGB}{230,242,255}
\definecolor{rowblueB}{RGB}{244,249,255}
\definecolor{bestblue}{RGB}{193,216,252}

\providecommand{\cmark}{\ding{51}}
\providecommand{\xmark}{\ding{55}}
\tikzset{flow/.style={-{Stealth[length=1.7mm]},draw=ink,line width=.65pt},
 control/.style={flow,dashed,draw=routeblue},
 block/.style={draw=ink!55,rounded corners=2pt,fill=white,align=center,inner sep=5pt,font=\sffamily\fontsize{8}{9}\selectfont},
 note/.style={font=\sffamily\scriptsize,align=center,text=ink}}

\title{OmniMoE-VL: A Sparse Vision-Language Model\\with Coupled Visual-Depth Routing}
\author{
Long Qian$^{1,2}$,
Bingke Zhu$^{1,2}$\thanks{Corresponding author.},
Jiaqi Wei$^{1,3}$\footnotemark[2],
Yingying Chen$^{1,2}$,
Jinqiao Wang$^{1,2,4}$
\\
$^{1}$Foundation Model Research Center, Institute of Automation, Chinese Academy of Sciences
\\
$^{2}$School of Future Technology, University of Chinese Academy of Sciences
\\
$^{3}$School of Engineering, Cardiff University
\\
$^{4}$Wuhan AI Research
\\
\texttt{qianlong2024@ia.ac.cn,}
\texttt{WeiJ16@cardiff.ac.uk,}
\\
\texttt{\{bingke.zhu,yingying.chen,jqwang\}@nlpr.ia.ac.cn}
}
\hypersetup{pdftitle={OmniMoE-VL: A Sparse Vision-Language Model with Coupled Visual-Depth Routing},pdfauthor={Anonymous authors}}

\begin{document}
\maketitle
\begin{abstract}
Vision-language models (VLMs) increasingly use sparse mixture-of-experts (MoE) to scale language-side computation, yet visual information is typically routed only after passing through a fixed cross-modal interface. This leaves an important decision unresolved: which intermediate visual representations should be exposed to language computation for a given question? We introduce OmniMoE-VL, a sparse VLM with a coupled visual-depth routed projector. For each image–prompt pair, the projector selects a sparse set of intermediate visual depths and reuses the resulting global preference to guide both local patch fusion and dynamic visual injection into the language model. This design enables question-dependent visual access while preserving the native visual-token sequence, and complements token-level expert routing in the vision and language stacks. Across eight image-based benchmarks, OmniMoE-VL achieves an average score of 85.9 with 28B total and 9B activated parameters. Controlled comparisons show that the routed visual interface provides the dominant architectural gain, while matched route and component controls, same-image route analysis, and route interventions further support the value of coupling and question-conditioned visual access.
\end{abstract}

\section{Introduction}
\label{sec:introduction}

Vision-language models (VLMs) connect a visual encoder to a language model through a multimodal interface, enabling image understanding, text recognition, and multimodal reasoning~\citep{liu2023visualinstructiontuning,li2024llavaonevisioneasyvisualtask}. Scaling this interface is increasingly important as VLMs are asked to read fine text, recognize objects, resolve spatial relations, and combine evidence across an image. Sparse mixture-of-experts (MoE)~\citep{lin2024moe,jing2025evomoeexpertevolutionmixture} offers a direct capacity-compute trade-off by assigning each token to a small subset of parameter experts. These models establish language-side conditional computation as a practical way to scale VLMs, but the visual representation delivered to that computation is still determined by the interface before the language experts are invoked.

The interface therefore controls a decision that language-side routing does not: which visual source representations become available to the model. Intermediate encoder depths carry different kinds of evidence. Earlier streams retain local appearance and text detail, middle streams organize neighboring patches into objects, and later streams provide broader contextual structure. A dense projector or a fixed visual stream merges this depth axis before the language router sees the representation. The language router can then choose how to process these tokens, but it cannot recover a visual stream that was not exposed at the interface. Even when visual tokens are prepended before the first LLM block, this remains a post-projection routing problem: the source-depth decision has been made.

Question-dependent source access also has two downstream destinations. The projector fuses aligned visual patches into the sequence that enters the language model, while deep visual injection supplies additional information to intermediate language states~\citep{meng2024deepstackdeeplystackingvisual}. If these destinations select depths independently, the input sequence and the layer updates can emphasize different visual evidence for the same image-question pair. We instead require a shared global preference over visual depths, then give each destination the operation it needs: local patch compatibility refines the preference into group-specific weights for fusion, whereas fixed depth-to-layer anchors use the global weights to scale dynamic residual updates. This separates global question-dependent selection from local spatial adaptation while keeping both pathways aligned. The resulting hypothesis is directly testable: different questions about one image should change the route, and replacing the native route with a route induced by another question should impair prediction.

\begin{figure}[t]
\centering
\includegraphics[width=\linewidth]{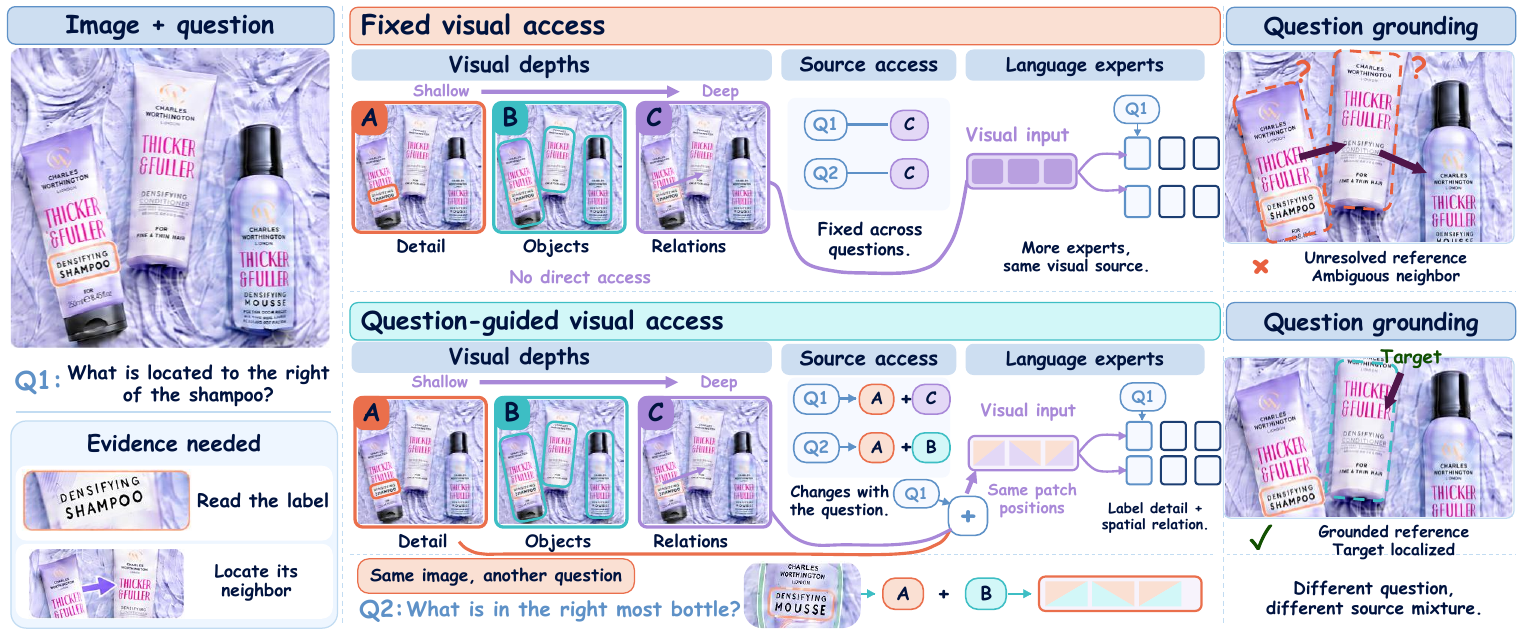}
\caption{\textbf{Question-dependent visual evidence motivates depth selection.} The first question requires label recognition and spatial context to identify the shampoo and its right-hand neighbor; the second requires the label on the rightmost bottle. A fixed interface supplies the same visual source to language experts across questions. Prompt-conditioned depth selection combines intermediate streams according to the evidence needed, connecting fine detail, object context, and spatial relations. Colors track the contributing visual streams.}
\label{fig:interface}
\end{figure}

To address these challenges, we introduce \method, a sparse VLM built around this routed interface (Figure~\ref{fig:interface}). The vision and language stacks use token-level parameter experts, while the projector performs sample-level routing over ten streams tapped at blocks $1,5,\ldots,37$. A prompt-conditioned bank of $16$ softly participating prototypes forms eight router queries; cross-attention over the projected streams produces a top-4 global route $w$. The route supplies the logarithmic prior for group-local top-4 weights $g_b$, whose aligned weighted fusion preserves the patch sequence, and the same $w$ weights dynamic transformations injected at fixed language layers $a(s)=s$. Thus the model has one explicit source-selection decision with two coordinated uses, rather than three unrelated sparse modules. An ordered usage prior and underuse bonus guide exploration during early training.

Across eight benchmarks, \method achieves strong performance with 28B total and 9B activated parameters, reaching an average score of 85.9. Under a common training recipe, introducing the routed visual interface improves the LLM-only MoE baseline by 4.9 points, with vision-side experts providing a further complementary gain. Matched controls show that coupling fusion and injection through a shared depth preference is more effective than learning the two routes independently. Beyond aggregate performance, the learned route changes with questions posed about the same image, and route interventions show that mismatched visual-depth preferences impair prediction. The projector also retains its advantage under image down-sampling and blur. Our contributions are:
\begin{itemize}
\item We analyze the source-access bottleneck created when sparse language computation receives a fixed or already fused visual representation, and quantify question-dependent visual access with same-image route comparisons and route interventions.
\item We introduce \method, a sparse vision-language model whose projector couples a prompt-conditioned global route visual depths with group-local patch fusion and dynamic language-layer injection, while token-level experts scale the vision and language stacks.
\item We evaluate \method\ on eight benchmarks and matched architectural controls, showing an 80.3$\rightarrow$85.2$\rightarrow$85.9 progression from LLM-only MoE to routed projector to the full model, together with a 1.1-point benefit from sharing the fusion and injection route.
\end{itemize}

\section{Related work}
\label{sec:related}

\paragraph{Sparse multimodal computation.}
MoE-LLaVA expands multimodal language-model capacity through sparse experts~\citep{lin2024moe}, while EvoMoE develops expert initialization and token-aware multimodal routing~\citep{jing2025evomoeexpertevolutionmixture}. These routers assign tokens to parameter experts. \method\ complements this computation with a projector that selects which intermediate visual streams contribute to the input sequence and language-layer updates.

\paragraph{Visual interfaces and deep injection.}
LLaVA-style models align image features with language embeddings through a projector~\citep{liu2023visualinstructiontuning,li2024llavaonevisioneasyvisualtask}, and BLIP-2 uses learned queries to bridge the backbones~\citep{li2023blip2}. DeepStack extends visual access by supplying visual tokens to deeper language layers~\citep{meng2024deepstackdeeplystackingvisual}. \method\ connects input fusion and deep injection through a route computed for each image-prompt pair. Prompt-conditioned queries compare visual streams, local fusion refines the resulting preference for each patch group, and the same preference weights residual injection at assigned language layers. This connection preserves patch positions while coordinating the two visual pathways.

\paragraph{Low-rank adaptation and routing structure.}
LoRA represents learned weight updates with low-rank factors~\citep{hu2022lora}. In \method, attention over the image and prompt constructs dynamic low-rank transformations, and the depth route weights their application. Query-participation control regulates the descriptors used for selection, while an ordered usage prior encourages exploration across visual depths during training.

\section{Method}
\label{sec:method}

\subsection{OmniMoE-VL overview}
\label{sec:overview}

\method\ combines token-level vision and language experts with a sample-level visual-depth projector. For image-prompt pair $x=(x_V,x_T)$, let $H_T\in\R^{L_T\times d}$ denote prompt embeddings and $V_s\in\R^{N\times d_V}$ a visual stream. We extract $S=10$ streams at blocks $\pi(s)=1+4(s-1)$, or $1,5,\ldots,37$, with aligned patch positions. Prompt-conditioned queries select four depths through $w(x)\in\Delta^{S-1}$, shared by visual input tokens $Z$ and language-layer transformations $A_\ell$:
\begin{align}
 Z(x)&=\mathcal P\big(V_{1:S},H_T;w(x)\big),
 \label{eq:interface}\\
 A_\ell(x)&=\sum_{s:a(s)=\ell}w_s(x)G_\ell T_s(x).
 \label{eq:coupling}
\end{align}
Here $\mathcal P$ fuses projected streams at corresponding patch positions into tokens that join the prompt at the language input. The image and prompt generate $T_s(x)$, and learned low-rank $G_\ell$ maps its output to layer $\ell$. Fixed anchors $a(s)=s$ connect the streams to language layers $1,\ldots,10$.

Figure~\ref{fig:method} traces three distributions: query weights $p$ determine the prompt descriptors; $w$ selects visual depths; and local weights $g_b$ refine this preference for patch group $b$. The same $w$ directly weights injection. Table~\ref{tab:routes} distinguishes these axes from token-level parameter-expert routing.

\begin{figure}[t]
\centering
\includegraphics[width=\linewidth]{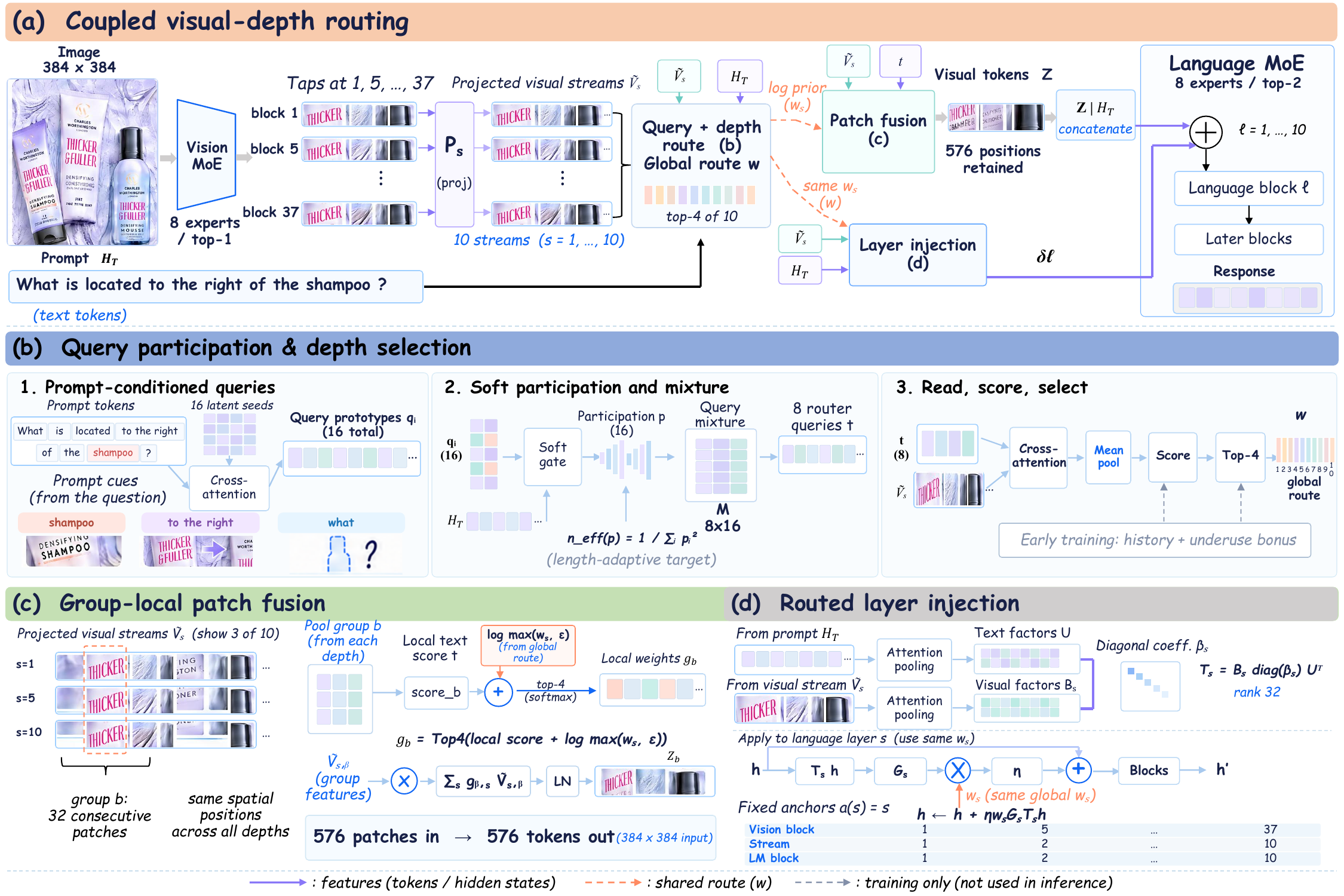}
\caption{\textbf{OmniMoE-VL couples visual-depth selection with fusion and language-layer updates.} (a) The model routes ten projected visual streams through one global depth distribution $w$. (b) Sixteen softly weighted prototypes form eight queries for top-4 depth selection. (c) The route supplies a prior for group-local weights $g_b$, which fuse aligned features into 576 visual tokens. (d) The same $w$ scales rank-32 updates at fixed language layers $a(s)=s$. Solid arrows carry features; dashed peach arrows carry shared routing weights.}
\label{fig:method}
\end{figure}

\subsection{Prompt-conditioned query participation}
\label{sec:queries}

A learned bank $C\in\R^{Q\times d}$ cross-attends to $H_T$, producing $Q=16$ prompt-conditioned prototypes $q_i$ with participation weights $p_i$:
\begin{equation}
 q_{1:Q}=\LN\big(\Attn(C,H_T,H_T)+\rho C\big),\qquad
 p=(1-Q\epsilon_q)\softmax(u(x_T))+\epsilon_q\mathbf 1.
 \label{eq:queries}
\end{equation}
The residual $\rho C$ retains each seed's contribution. Logits $u$ combine prototype features, prompt context, and query biases; $\epsilon_q=10^{-4}$ keeps each prototype available. Prompt length sets a smooth participation target:
\begin{equation}
 \neff(p)=\frac{1}{\sum_i p_i^2},\qquad
 \mathcal L_q=\big(\neff(p)-\ntgt(L_T)\big)^2,\qquad
 \ntgt(L_T)=\operatorname{clip}\big(\log(1+L_T)+4,2,Q\big).
 \label{eq:participation}
\end{equation}
The effective count increases as probability mass spreads across the bank. Longer prompts receive a larger participation target within the same fixed capacity. A learned $M\in\R^{R\times Q}$ uses the soft weights $p$ to form $R=8$ router queries:
\begin{equation}
 \alpha_{ri}=\softmax_i\!\left(\frac{M_{ri}+\log p_i}{\tau_q}\right),
 \qquad t_r=\sum_i\alpha_{ri}q_i,\qquad \tau_q=0.9.
 \label{eq:querymix}
\end{equation}
These mixtures retain multiple prompt-conditioned descriptors for the visual-depth comparison that follows. All prototypes participate through soft weights, while the bank size remains fixed.

\subsection{Global depth selection and local patch fusion}
\label{sec:router}

Depth-specific projections $\widetilde V_s=P_s(V_s)\in\R^{N\times d}$ let router queries read all streams at language width. A scoring network maps each mean readout to its depth score:
\begin{equation}
 d_s(x)=\frac{1}{R}\sum_{r=1}^{R}\Attn_s(t_{1:R},\widetilde V_s,\widetilde V_s)_r,
 \qquad r_s(x)=f_s(d_s(x)).
 \label{eq:depthscore}
\end{equation}
At inference, the resulting scores determine the global route, $w=\TopK_4(r/T)$. Here $\TopK_k$ retains the $k$ largest logits, applies softmax on that support, and sets the remaining entries to zero.

During early training, we smooth historical usage across neighboring depths and add an underuse bonus. With $c_s$ the exponential moving average of depth usage, the augmented logits are
\begin{equation}
 K_{ij}=e^{-(i-j)^2/\gamma},\quad
 \bar p=\frac{K(c+\alpha\mathbf1)}{\mathbf1^\top K(c+\alpha\mathbf1)},\quad
 z_s=\frac{r_s}{T}+\lambda_h\log\bar p_s+b_s(c).
 \label{eq:history}
\end{equation}
The kernel spreads historical mass to nearby depths, while $b_s$ favors underused streams through a normalized inverse-square-root count bonus. Both additions are disabled after exploration and during evaluation, when content scores determine the route. Appendix~\ref{app:training} specifies the exploration schedule and auxiliary objectives for depth routing.

To refine the global preference locally, we partition patches into consecutive groups $\mathcal B_b$ of 32 tokens. Let $v_{b,s}$ be the normalized group mean of $\widetilde V_s$ and $\bar t$ the normalized mean of $t_{1:R}$. Local compatibility and the global prior jointly determine fusion:
\begin{align}
 g_b&=\TopK_4\!\left(\left\{\frac{\langle v_{b,s},\bar t\rangle}{\tau_g}
       +\log\max(w_s,\epsilon_w)\right\}_{s=1}^{S}\right),
 \label{eq:localroute}\\
 Z_{\mathcal B_b}&=\LN\!\left(\sum_{s=1}^{S}g_{b,s}\widetilde V_{s,\mathcal B_b}\right),
 \qquad \tau_g=0.7,\quad\epsilon_w=10^{-12}.
 \label{eq:localfusion}
\end{align}
Each group uses its own $g_b$ to refine the sample-wide preference $w$, and concatenating the fused groups restores the original patch sequence. The log floor assigns a small local prior to globally inactive depths. Their global weights remain zero for the injection branch, which uses $w$ directly.

\subsection{Routed visual injection into language layers}
\label{sec:injection}

Alongside the fused input, the selected streams provide visual updates to intermediate language states. Attention over the prompt constructs directions $U(x_T)\in\R^{d\times r_T}$, and attention over each projected stream constructs visual directions $B_s(x_V)\in\R^{d\times r_T}$. Pairing the text and visual directions gives coefficients $\beta_s(x)$ and the transformation
\begin{equation}
 T_s(x)=B_s(x_V)\diag(\beta_s(x))U(x_T)^\top,\qquad r_T=32.
 \label{eq:dynamic}
\end{equation}
Applied to a hidden state, these factors first read prompt-derived coordinates and then reconstruct a visual residual through $B_s$. A learned rank-32 adapter $G_{a(s)}$ maps the residual to its assigned language layer. Equation~\ref{eq:coupling} weights this update by the same $w_s$ that guides projector fusion, linking the two uses of stream $s$.

For a column-vector hidden state $h_{\ell,t}$, the update is
\begin{equation}
 \delta_{\ell,t}=\eta_{\ell,t}\,A_\ell(x)h_{\ell,t},\qquad
 H_{\ell+1}=\Phi_\ell(H_\ell+\Delta_\ell).
 \label{eq:injection}
\end{equation}
Here $\eta_{\ell,t}$ combines injection gain, token-type weighting, and residual-norm control. Factorized evaluation avoids a dense $d\times d$ matrix. The update enters before the assigned block, with $w$, $U$, $B_s$, and $\beta_s$ determined by the image and prompt.
\section{Architectural Analysis}
\label{sec:analysis}

\subsection{Global selection, local adaptation, and layer-wise access}
Each patch group refines the common depth preference through $g_b$. Write $c_{b,s}=\langle v_{b,s},\bar t\rangle$ and $\bar w_s=\max(w_s,\epsilon_w)$. For two depths retained by local top-$k$ selection,
\begin{equation}
 \log\frac{g_{b,i}}{g_{b,j}}
 =\log\frac{\bar w_i}{\bar w_j}
  +\frac{c_{b,i}-c_{b,j}}{\tau_g}.
 \label{eq:local-odds}
\end{equation}
The shared prior sets the depth preference, and local compatibility adjusts it across patch groups. Injection carries this preference into language computation: $w_s$ scales the prompt--stream transformation at anchor $a(s)$, whose destination is fixed and content is input-dependent.

\subsection{Learning a route through both visual pathways}
Both pathways train $w$ through the language-modeling loss. Holding the selection support fixed, its effects on visual tokens $Z$ and transformations $A_\ell$ give
\begin{equation}
 \nabla_w\mathcal L_{\mathrm{LM}}
 =J_{Z,w}^{\top}\nabla_Z\mathcal L_{\mathrm{LM}}
 +\sum_{\ell}J_{A_\ell,w}^{\top}\nabla_{A_\ell}\mathcal L_{\mathrm{LM}}.
 \label{eq:shared-credit}
\end{equation}
The local Jacobians hold other inputs fixed. Table~\ref{tab:public}(b) tests shared versus independent routing during training; Table~\ref{tab:omni-route-intervention}(b) tests each pathway's use of the learned route at inference.

\subsection{Capacity and computation}
Fusion reads ten streams from one vision-encoder pass and retains the 576-position patch grid. Depth selection therefore changes features without extending the language input. Injection applies its two factors to $L$ states at cost $O(Ld(r_T+r_G))$, in addition to constructing them. Before token-dependent gains, their composition has rank at most $\min(r_T,r_G)=32$ per anchor.

The model stores 28B parameters and activates 9B. Execution also includes dense attention, visual processing, and dispatch, while storage covers the full expert pool. Section~\ref{sec:cost} reports memory and throughput; Appendix~\ref{app:analysis} develops the mathematical properties.

\section{Experiments}
\label{sec:experiments}
\subsection{Experimental setup}
\label{sec:setup}

\paragraph{Model and training.}
We initialize the vision encoder from SigLIP2-Giant at $384\times384$ resolution and the language model from Qwen3-8B~\citep{tschannen2025siglip2,yang2025qwen3}. Their even-indexed blocks contain eight feed-forward experts, using top-1 routing in the final vision stack and top-2 in the language stack. The projector connects the backbones through ten visual depths, 16 query prototypes, eight router queries, and top-4 depth selection. Table~\ref{tab:training} gives the training configuration, and Appendix~\ref{app:implementation} details the routing and injection settings.

\begin{table}[t]
\centering
\caption{Training configs for \textit{OmniMoE-VL}.}
\label{tab:training}
\small
\resizebox{0.9\linewidth}{!}{
\begin{tabular}{lcccc}
\toprule
Config & Stage I & Stage II & Stage III & Stage IV \\
\midrule
Vision MoE (experts / top-$k$) &
- &
- &
8 / 2 &
8 / 1 \\
LLM MoE (experts / top-$k$) &
- &
8 / 2 &
8 / 2 &
8 / 2 \\
Projector MoE (scales / top-$k$) &
10 / 4 &
10 / 4 &
10 / 4 &
10 / 4 \\
Trainable MoE params &
Projector only &
LLM only &
Vision only &
Projector + LLM + Vision \\
\midrule
Data & \textit{OMVL-16M-Alignment} & \textit{OMVL-16M-LLM} & \textit{OMVL-16M-Vision} & \textit{OMVL-16M-Joint} \\
Image resolution &
\multicolumn{4}{c}{$384\times384$} \\
Base Vision Tower &
\multicolumn{4}{c}{siglip2-giant-opt-patch16-384} \\
Base LLM Model &
\multicolumn{4}{c}{Qwen3-8B} \\
Epochs & \multicolumn{4}{c}{$1$} \\
Learning rate &
$1e-4$ &
$5e-5$ &
$5e-5$ &
$1e-5$ \\
LR schedule &
\multicolumn{4}{c}{$cosine$} \\
Batch size per GPU &
2 & 1 & 1 & 1 \\
Precision &
\multicolumn{4}{c}{BF16} \\
Deepspeed's ZeRO config &
ZeRO-2 &
ZeRO-2 offload &
ZeRO-2 offload &
ZeRO-2 offload \\
\bottomrule
\end{tabular}}
\end{table}

\begin{table}[t]
\centering
\caption{Eight-benchmark comparison. Parameters are in billions; $VQA^T$/$VQA^D$ denote TextVQA/DocVQA. Avg. uses listed scores; superscripts give their count. $\dagger$ marks independent evaluations (Appendix~\ref{app:table2-sources}). Bold/underlined scores denote first/second place.}
\label{tab:main-vlm}
\small
\setlength{\tabcolsep}{3pt}
\resizebox{\linewidth}{!}{
\rowcolors{4}{rowblueA}{rowblueB}
\begin{tabular}{l c c c c c c c c c c c}
\toprule
\multirow{2}{*}{\textit{Methods}} &
\multirow{2}{*}{\textit{\# Params}} &
\multirow{2}{*}{\textit{\# Act.}} &
\multicolumn{8}{c}{\textit{Benchmarks}} &
\multirow{2}{*}{\textit{Avg.}} \\
\cmidrule(lr){4-11}
& & & \textit{MMBench} & \textit{MMVet} & \textit{$VQA^T$} & \textit{$VQA^D$} & \textit{POPE} & \textit{MMMU} & \textit{OCRBench} & \textit{ChartQA} & \\
\midrule
\multicolumn{12}{l}{\textbf{\textit{Dense models}}} \\
Qwen3-VL-8B-Instruct & 8B & 8B & 85.0 & - & 82.9$^{\dagger}$ & - & - & 69.6 & \textbf{89.6} & \textbf{89.6} & 83.3$^{(5)}$ \\
GLM-4.6V-Flash & 9B & 9B & \underline{86.9} & - & - & - & - & \underline{71.1} & 84.7 & - & 80.9$^{(3)}$ \\
Qwen2.5-VL-7B & 7B & 7B & 83.5 & 67.1 & 84.9 & \underline{95.7} & 86.4 & 58.6 & 86.4 & 87.3 & 81.2$^{(8)}$ \\
\multicolumn{12}{l}{\textbf{\textit{Sparse models}}} \\
Step-3 & 321B & 38B & 81.1$^{\dagger}$ & \textbf{79.4}$^{\dagger}$ & - & - & - & \textbf{74.2} & 83.7$^{\dagger}$ & - & 79.6$^{(4)}$ \\
MoE-LLaVA & 5B & 4B & 68.0 & 35.9 & 57.0 & 25.6$^{\dagger}$ & 85.7 & - & - & 15.4$^{\dagger}$ & 47.9$^{(6)}$ \\
EvoMoE & 7B & 7B & 65.8 & - & 53.8 & - & 87.3 & - & - & - & 69.0$^{(3)}$ \\
\midrule
\textbf{\textit{OmniMoE-VL (w/o. Vision MoE)}} & \textbf{25B} & \textbf{9B} & \underline{86.9} & 77.9 & \underline{86.1} & \underline{95.7} & \textbf{88.3} & 69.4 & 88.2 & \underline{89.3} & \underline{85.2}$^{(8)}$ \\
\textbf{\textit{OmniMoE-VL}} & \textbf{28B} & \textbf{9B} & \textbf{87.8} & \underline{78.3} & \textbf{87.9} & \textbf{96.1} & \underline{88.2} & 70.1 & \underline{89.3} & \textbf{89.6} & \textbf{85.9}$^{(8)}$ \\
\bottomrule
\end{tabular}}
\end{table}

\paragraph{Data.}
We train on OMVL-16M, approximately 16.2M multimodal instruction samples curated from LLaVA-NeXT, Infinity-MM, LLaVA-OneVision, and LLaVA-Instruct~\citep{liu2024llavanext,gu2024infinitymmscalingmultimodalperformance,li2024llavaonevisioneasyvisualtask,liu2023visualinstructiontuning}. Appendix~\ref{app:data-construction} describes benchmark de-duplication. Stack-wise variants share these data and the progressive recipe. A second setting uses the public MoE-LLaVA/EvoMoE data and recipe; each component replacement retains the backbone, query bank, visual depths, routing budget, and remaining modules.

\paragraph{Progressive optimization.}
Stage I aligns the backbones through the projector. Stages II and III train language and vision experts in turn, and Stage IV jointly tunes all stacks at a lower LR.

\paragraph{Benchmarks and metrics.}
Our eight benchmarks are MMBench, MM-Vet, TextVQA, DocVQA, POPE, MMMU, OCRBench, and ChartQA~\citep{liu2024mmbenchmultimodalmodelallaround,yu2024mmvetevaluatinglargemultimodal,singh2019towards,mathew2020docvqa,li2023evaluating,yue2024mmmu,liu2024ocrbench,masry2022chartqa}, covering multimodal understanding, text recognition, documents, object hallucination, multidisciplinary reasoning, and charts. We express all scores on a 0-100 scale, normalizing OCRBench by its maximum, and report each mean with its task count in Table~\ref{tab:main-vlm}. Percentage changes use the stated comparison score as their reference. Component ablations-removal and corruption-robustness experiment experiments use MMBench, MM-Vet, TextVQA, DocVQA, and POPE.

\subsection{Comparison with dense and sparse VLMs}
\label{sec:comparison}

Table~\ref{tab:main-vlm} places \method\ alongside dense and sparse VLMs using model releases and published evaluations~\citep{qwen3vl2025,glm46v,bai2025qwen25vltechnicalreport,step3system,lin2024moe,jing2025evomoeexpertevolutionmixture}. Its eight-task average of 85.9 exceeds Qwen2.5-VL-7B's 81.2 by 4.7 points (+5.8\%), with gains on every task. The largest are 11.5 points on MMMU (+19.6\%) and 11.2 on MM-Vet (+16.7\%). \method\ leads the listed MMBench, TextVQA, and DocVQA results, matches Qwen3-VL-8B on ChartQA, and gains 5.0 points (+6.0\%) over it on TextVQA. Step-3 leads MMMU and MM-Vet at 74.2 and 79.4. The following controls examine the architectural sources of \method's performance.

\begin{table}[t]
\centering
\caption{Public-recipe results and matched projector controls. Left: comparison scores (benchmark coverage in Appendix~\ref{app:external}). Right: shared versus independent routes at matched parameter count, followed by component replacements. Visual experts are disabled throughout (b).}
\label{tab:public}
\small
\begin{minipage}[t]{0.45\linewidth}
    \centering
    \small
    \textbf{(a) Comparison under public recipes}
    \vspace{0.2em}

    \resizebox{\linewidth}{!}{
    \rowcolors{2}{rowblueA}{rowblueB}
    \begin{tabular}{l c}
      \toprule
      \textit{Method} & \textit{Avg.} \\
      \midrule
      MoE-LLaVA~\citep{lin2024moe} & 61.7 \\
      EvoMoE~\citep{jing2025evomoeexpertevolutionmixture} & 69.0 \\
      \midrule
      \textit{OmniMoE-VL (w/o. Visual stack MoE)} & \underline{83.7} \\
      \textbf{\textit{Full OmniMoE-VL}} & \textbf{84.1} \\
      \bottomrule
    \end{tabular}}
  \end{minipage}
\hfill
\begin{minipage}[t]{0.53\linewidth}
    \centering
    \textbf{(b) Matched projector controls}

    \resizebox{\linewidth}{!}{
    \rowcolors{2}{rowblueA}{rowblueB}
    \begin{tabular}{l c}
      \toprule
      \textit{Variant} & \textit{Avg.} \\
      \midrule
      \textbf{\textit{Full coupled projector}} &
      \textbf{83.7} \\
      \textit{Independent fusion/injection routes} &
      82.6 {\small (-1.1)} \\
      \textit{Plain softmax gate + matched reg} &
      81.8 {\small (-1.9)} \\
      \textit{Standard router + load-balance/UCB} &
      83.1 {\small (-0.6)} \\
      \textit{Per-layer LoRA (matched rank/params)} &
      \underline{83.3} {\small (-0.4)} \\
      \bottomrule
    \end{tabular}}
  \end{minipage}
\end{table}

\begin{table}[t]
\centering
\caption{\textbf{Architectural ablations.} (a) Stack placement across eight benchmarks: V/P/L denote vision/projector/language experts; $\ddagger$ marks the aggregate-derived mean (Appendix~\ref{app:aggregation}). (b) Projector removals under the component protocol: Q/R/I denote query control, depth routing, and visual injection.}
\label{tab:architecture}
\label{tab:stack}
\label{tab:components}
\normalsize
\setlength{\tabcolsep}{3pt}
\renewcommand{\arraystretch}{1.16}
\begin{minipage}[t]{0.46\linewidth}
\centering
\textbf{(a) Stack placement}\par\smallskip
\rowcolors{2}{rowblueA}{rowblueB}
\begin{tabular}{lcccc}
      \toprule
      \textit{Variant} & \textit{V} & \textit{P} & \textit{L} & \textit{Avg.} \\
      \midrule
      \textit{LLM-only MoE} &
      \xmark & \xmark & \cmark & $80.3^\ddagger$ \\
      \textit{LLM + Projector MoE} &
      \xmark & \cmark & \cmark & \underline{85.2} \\
      \textbf{\textit{Full OmniMoE-VL}} &
      \cmark & \cmark & \cmark & \textbf{85.9} \\
      \bottomrule
    \end{tabular}
\end{minipage}\hfill
\begin{minipage}[t]{0.52\linewidth}
\centering
\textbf{(b) Projector components}\par\smallskip
\rowcolors{2}{rowblueA}{rowblueB}
\begin{tabular}{lcccc}
      \toprule
      \textit{Variant} & \textit{Q} & \textit{R} & \textit{I} & \textit{Avg.} \\
      \midrule
      \textit{Full} &
      \cmark & \cmark & \cmark & \textbf{87.6} \\
      \textit{MLP-only projector} &
      \xmark & \xmark & \xmark & 81.3 \\
      \textit{w/o. Q} &
      \xmark & \cmark & \cmark & \underline{84.3} \\
      \textit{w/o. I} &
      \cmark & \cmark & \xmark & 83.8 \\
      \textit{w/o. Q\&R} &
      \xmark & \xmark & \cmark & 83.1 \\
      \bottomrule
    \end{tabular}
\end{minipage}
\end{table}

Under the public recipe (Table~\ref{tab:public}), the projector reaches 83.7, and vision experts raise the score to 84.1, a 0.5\% relative gain. The accompanying controls examine routing coupling and the projector's individual operations.

\subsection{Stack placement and projector components}
\label{sec:stack-results}
\label{sec:components}

Table~\ref{tab:stack}(a) places most of the gain at the visual interface. The routed projector raises the eight-task average from 80.3 to 85.2 (+6.1\%), adding 4.7 points on MMMU, 5.0 on OCRBench, and 1.2 on ChartQA. Vision experts add a further 0.7 average points (+0.8\%), bringing the total gain to 7.0\%. They raise TextVQA from 86.1 to 87.9 and OCRBench from 88.2 to 89.

Removing query control lowers the score from 87.6 to 84.3 (3.8\%); removing injection gives 83.8 (4.3\%). Removing both query control and structured routing gives 83.1 (5.1\%; Table~\ref{tab:components}(b)). Both selecting depths and applying their transformations contribute to the projector's gain.

At matched parameter count, independent fusion/injection routes score 82.6 vs. 83.7 for the coupled projector (Table~\ref{tab:public}(b)). Sharing the depth preference therefore improves the average by 1.1 points.

The softmax replacement removes the length-adaptive target while retaining the query synthesizer, 16 prototypes, ten depths, top-4 routing, and entropy, diversity, and balancing terms. With the same validation split and tuning budget, it scores 81.8, trailing the full projector by 1.9 points (+2.3\%). The standard load-balanced/UCB router scores 83.1, and rank- and parameter-matched per-layer LoRA scores 83.3; the proposed operations improve over these alternatives by 0.7\% and 0.5\%.

\subsection{Question-conditioned visual access}
\label{sec:depth-sensitivity}

\begin{table}[t]
\centering
\caption{\textbf{Question-conditioned visual access.}
For same-image comparisons, Qwen3-VL depth-profile distances are measured in nats/token,
while \method\ route distances use total variation (TV).
$\Delta D=D_{\mathrm{different}}-D_{\mathrm{paraphrase}}$.
The lower block reports route interventions in \method\ while holding the image,
target question, and non-route inputs fixed.}
\label{tab:question-conditioning}
\label{tab:depth-sensitivity}
\label{tab:omni-route-variation}
\label{tab:omni-route-intervention}

\small
\setlength{\tabcolsep}{3.5pt}
\renewcommand{\arraystretch}{1.16}
\begin{tabular}{@{}lrrcccccc@{}}
\toprule
\multirow{2}{*}{Comparison}
& \multirow{2}{*}{Images}
& \multirow{2}{*}{Targets}
& \multicolumn{3}{c}{Qwen3-VL: $D_{\mathrm{prof}}$}
& \multicolumn{3}{c}{\method: $D_{\mathrm{TV}}$} \\
\cmidrule(lr){4-6}\cmidrule(lr){7-9}
& & &
Different & Paraphrase & $\Delta D$
& Different & Paraphrase & $\Delta D$ \\
\midrule
\rowcolor{rowblueA}
Fixed target
& 283 & 322
& 0.257 & 0.171 & 0.086
& 0.301 & 0.162 & 0.139 \\
\rowcolor{rowblueB}
+ matched structural type
& 136 & 174
& 0.251 & 0.169 & 0.083
& 0.282 & 0.159 & 0.123 \\
\rowcolor{rowblueA}
+ matched answer length
& 193 & 214
& 0.253 & 0.193 & 0.060
& 0.264 & 0.171 & 0.093 \\
\rowcolor{rowblueB}
+ identical reference answer
& 44 & 53
& 0.137 & 0.089 & 0.049
& 0.217 & 0.145 & 0.072 \\
\midrule
\multicolumn{9}{@{}l}{\textit{Route intervention: reference-answer $\Delta$NLL (nats/token)}} \\
\cmidrule(lr){1-9}
\multicolumn{4}{@{}l}{Fusion route}
& \multicolumn{4}{l}{Injection route}
& $\Delta\mathrm{NLL}\downarrow$ \\
\midrule
\rowcolor{rowblueA}
\multicolumn{4}{@{}l}{Native $w_A$}
& \multicolumn{4}{l}{Native $w_A$}
& 0.00 \\
\rowcolor{rowblueB}
\multicolumn{4}{@{}l}{Paraphrase $w'_A$}
& \multicolumn{4}{l}{Paraphrase $w'_A$}
& 0.01 \\
\rowcolor{rowblueA}
\multicolumn{4}{@{}l}{Other-question $w_B$}
& \multicolumn{4}{l}{Native $w_A$}
& 0.04 \\
\rowcolor{rowblueB}
\multicolumn{4}{@{}l}{Native $w_A$}
& \multicolumn{4}{l}{Other-question $w_B$}
& 0.03 \\
\rowcolor{rowblueA}
\multicolumn{4}{@{}l}{Other-question $w_B$}
& \multicolumn{4}{l}{Other-question $w_B$}
& 0.09 \\
\bottomrule
\end{tabular}
\end{table}

We use 397 GQA images with four questions each~\citep{hudson2019gqa}, including 322 equivalent wordings across 283 images with the same semantic program and answer. Each target is paired with its own equivalent wording and with eligible other questions about the same image. The three controls separately match structural question type, reference-answer token count, or exact reference answer between the target and other question. We average eligible comparisons within targets, then targets within images, and weight images equally.

\paragraph{Visual-depth contributions.}
In Qwen3-VL-30B-A3B-Instruct~\citep{qwen3vl2025}, we zero one DeepStack output from encoder index 8, 16, or 24 (zero-based), retaining the other paths and main visual input. Evaluation uses BF16 and at most 262,144 pixels. Each entry of $e(q)$ is the change in reference-token NLL, excluding EOS. We center this profile to compare relative depth contributions:
\begin{equation}
 \widetilde e(q)=e(q)-\tfrac{1}{3}\big(\mathbf1^\top e(q)\big)\mathbf1,
 \qquad D_{\mathrm{prof}}(q,q')=\frac{\|\widetilde e(q)-\widetilde e(q')\|_2}{\sqrt{3}}.
 \label{eq:depth-sensitivity}
\end{equation}
Different questions have more distinct profiles than equivalent wordings in all four comparisons (Table~\ref{tab:depth-sensitivity}(a)), linking visual-depth contributions to the question's information demand.

\paragraph{Route variation in \method.}
We next measure how the model's global route changes with the question. For the same image, $w(q)$ is a distribution over ten depths with top-4 support. Its total variation distance is
\begin{equation}
 D_{\mathrm{TV}}(q_i,q_j)=\tfrac12\|w(q_i)-w(q_j)\|_1\in[0,1].
 \label{eq:route-tv}
\end{equation}
Table~\ref{tab:omni-route-variation}(a) gives distances of 0.301 for other questions and 0.162 for paraphrases, a paired gap of 0.139. Matching structural type, answer length, or the exact answer preserves this ordering, connecting route variation to question content beyond these controls.

The identical-answer control compares different questions with the same reference answer: route distances are 0.217 versus 0.145, a gap of 0.072. Thus, the distinction persists even when the output is fixed. Structural-type and answer-length matching preserve gaps of 0.123 and 0.093, respectively.

\paragraph{Route interventions.}
For target $q_A$, we replace its native route $w_A$ with a same-image paraphrase route $w'_A$ or another question's route $w_B$. Only routing changes; the image, target prompt, queries, and operator factors remain fixed. Writing fusion and injection routes as $u,v$, we score reference answer $Y$ by
\begin{equation}
 \mathrm{NLL}(q;u,v)=-\frac{1}{|Y|}\sum_{t=1}^{|Y|}\log p(y_t\mid y_{<t},x,q;u,v).
 \label{eq:route-nll}
\end{equation}
We report $\Delta\mathrm{NLL}=\mathrm{NLL}(q_A;u,v)-\mathrm{NLL}(q_A;w_A,w_A)$ in nats per answer token. All conditions score the same reference answer with teacher forcing. Local fusion is recomputed under the transferred prior, while injection retains the target question's operator factors. Fusion-only, injection-only, and joint replacement increase NLL by 0.04, 0.03, and 0.09, versus 0.01 for paraphrase transfer (Table~\ref{tab:omni-route-intervention}(b)). Matching routes to the question benefits prediction through both pathways.

\subsection{Visual degradation and deployment cost}
\label{sec:robustness}
\label{sec:cost}

Under down-sampling and blur (Table~\ref{tab:robustness}(a)), LLM-only MoE loses 3.9 points (4.8\%), the projector variant 1.2 (1.4\%), and the full model 0.8 (0.9\%). The interface's gains therefore persist with degraded input, with vision experts adding further robustness.

\begin{table}[h]
\centering
\caption{\textbf{Robustness and deployment.} (a) Clean/degraded scores under down-sampling and blur; drop is relative to clean accuracy. (b) Peak VRAM and decoding throughput on two H800 GPUs, BF16, $384\times384$ images (Appendix~\ref{app:profile}).}
\label{tab:deployment}
\label{tab:robustness}
\label{tab:profile-main}
\normalsize
\setlength{\tabcolsep}{3pt}
\renewcommand{\arraystretch}{1.16}
\begin{minipage}[t]{0.51\linewidth}
\centering
\textbf{(a) Visual degradation}\par\smallskip
\rowcolors{2}{rowblueA}{rowblueB}
\begin{tabular}{lccc}
      \toprule
      \textit{Variant} & \textit{Clean} & \textit{Shifted} & \textit{Drop} \\
      \midrule
      \textit{LLM-only MoE} &
      81.3 & 77.4 & 4.8\% \\
      \textit{LLM + Projector MoE} &
      \underline{87.0} & \underline{85.8} & \underline{1.4\%} \\
      \textbf{\textit{Full OmniMoE-VL}} &
      \textbf{87.6} & \textbf{86.8} & \textbf{0.9\%} \\
      \bottomrule
    \end{tabular}
\end{minipage}\hfill
\begin{minipage}[t]{0.47\linewidth}
\centering
\textbf{(b) Inference profile}\par\smallskip
\rowcolors{2}{rowblueA}{rowblueB}
\begin{tabular}{l c c}
\toprule
\textit{Model} & \shortstack{\textit{VRAM}\\\textit{(GB)}} & \shortstack{\textit{Decode}\\\textit{(tokens/s)}} \\
\midrule
Qwen3-VL-8B & \textbf{11} & \textbf{19.71} \\
\textbf{\textit{OmniMoE-VL}} & \underline{52} & \underline{17.52} \\
\bottomrule
\end{tabular}
\end{minipage}
\end{table}

\method\ decodes at 17.52 tokens/s with 52\,GB peak VRAM (Table~\ref{tab:profile-main}(b)): throughput is 11.1\% lower and memory 372.7\% higher than Qwen3-VL-8B. These costs include the full 28B expert storage, intermediate streams, and dynamic injection. ~\ref{app:utilization} reports expert and depth utilization.

\section{Conclusion}
\label{sec:conclusion}

In this work, we introduce \method, a sparse vision-language model that extends conditional computation from parameter-expert routing to prompt-conditioned visual access. \method\ combines sparse experts in the vision and language stacks with a routed projector that selects intermediate visual depths and coordinates their use in language computation. Across eight benchmarks, the model achieves an average score of 85.9 with 28B total and 9B activated parameters. Stack-wise comparisons show that the routed projector provides the largest architectural gain, improving the LLM-only MoE baseline by 4.9 points, while vision experts provide a further 0.7-point improvement. Component ablations and matched replacements support the contributions of query participation, depth routing, and visual injection, while image-degradation experiments show improved robustness. Overall, these results support prompt-conditioned visual-depth routing as a complementary form of conditional computation for sparse VLMs.
\label{end:main}

\bibliography{references}
\bibliographystyle{iclr2027_conference}
\clearpage
\appendix
\begin{center}
    {\Large\bf APPENDIX}
\end{center}
\appendix
\section{Implementation details}
\label{app:implementation}

\subsection{Tensor flow and routing granularity}
For a batch of $B$ image--prompt pairs, the projector receives prompt embeddings $[B,L_T,d]$, a valid-token mask $[B,L_T]$, and ten visual tensors $[B,N,d_V]$. Query synthesis converts the prompt embeddings into prototypes $[B,16,d]$ and participation weights $[B,16]$. Their weighted mixtures $[B,8,d]$ read the visual streams to compute $w\in\R^{B\times10}$, which is passed to both patch fusion and the injection provider.

Patch fusion operates on projected features of shape $[B,10,N,d]$, assigning a top-4 depth mixture $g_b$ to each group of 32 consecutive patch positions. Concatenating the fused groups restores an output of shape $[B,N,d]$. Normalization and a learned output scale prepare these tokens for insertion into the language-model input. At $384\times384$ resolution with $16\times16$ patches, the image grid contains $24\times24=576$ positions arranged in 18 groups.

For injection, the provider constructs $U\in\R^{B\times d\times32}$, ten visual-direction matrices of the same shape, and ten coefficient vectors $\beta_s\in\R^{B\times32}$. It applies these factors to the hidden states before the assigned Transformer block. With row-oriented hidden states, the operation is
\begin{equation}
 T_s(H)=\big((HU)\odot\beta_s\big)B_s^\top,
\end{equation}
where multiplication by $\beta_s$ broadcasts across hidden-state positions. After layer adaptation, $w_s$ scales the residual, and token-dependent gain and norm control adjust its magnitude before it enters the block.

\begin{table}[h]
\caption{Current full-model implementation. Visual-depth routing and parameter-expert routing operate on different axes.}
\label{tab:implementation}
\centering\small
\begin{tabularx}{\linewidth}{@{}lX@{}}
\toprule
Item & Setting \\
\midrule
Vision backbone & SigLIP2-Giant, patch size 16, input $384\times384$ \\
Language backbone & Qwen3-8B; hidden width $d=4096$ \\
Visual depth taps & $1,5,9,13,17,21,25,29,33,37$ (one-based) \\
Vision parameter experts & 8 experts, top-1, even-indexed blocks \\
Language parameter experts & 8 experts, top-2, even-indexed blocks \\
Query bank & $Q=16$; probability floor $10^{-4}$ \\
Router text queries & $R=8$; mixture temperature $0.9$ \\
Global visual route & $S=10$ depths; top-4 per sample \\
Projector output & Native patch sequence; groups of 32 tokens \\
Local visual route & Top-4 per group; temperature $0.7$; log prior from $w$ \\
Injection anchors & Stream index $s$ to language layer $a(s)=s$ \\
Dynamic / adapter ranks & $r_T=32$, $r_G=32$ \\
Token-type injection gain & Visual tokens: 1; text tokens: 0.25 when the mask is available \\
Factor regularization & Coefficient $10^{-5}$ on low-rank adapter factors \\
\bottomrule
\end{tabularx}
\end{table}

\begin{table}[h]
\caption{Routing objects and the decisions they control. Query participation is soft; visual-depth and parameter-expert selection use top-$k$ routes.}
\label{tab:routes}
\centering\small
\begin{tabular}{@{}llll@{}}
\toprule
Route & Unit & Candidates & Role \\
\midrule
$p$ & Prompt & 16 queries & Query participation \\
$w$ & Image--prompt pair & 10 depths, top-4 & Shared global depth selection \\
$g_b$ & Patch group & 10 depths, top-4 & Local feature fusion \\
Vision experts & Token & 8 experts, top-1 & Visual feature processing \\
Language experts & Token & 8 experts, top-2 & Language-state processing \\
\bottomrule
\end{tabular}
\end{table}

\subsection{Query gate and dynamic factors}
The query gate combines a prototype score, prototype--prompt-context interaction, prototype--latent-seed alignment, and a learned query bias. Its logits also contain pooled-text and length scalars shared across queries. Since softmax is invariant to these shared shifts, the participation objective supplies the explicit length adaptation. Attention and prompt pooling both exclude padding tokens.

To construct the dynamic injection, independent learned rank banks attend to prompt tokens and projected visual tokens, producing $U$ and $B_s$. The similarity between paired directions $u_j$ and $b_{s,j}$ then determines the coefficient vector:
\begin{equation}
 \beta_{s,j}=\softmax_j\!\left(\frac{u_j^\top b_{s,j}}{\sqrt d}\right).
\end{equation}
The resulting transformation feeds the learned low-rank adapter $G_\ell$, whose shared parameters act on the visual residual constructed for the current image--prompt pair. Residual-norm controls act before and after aggregation when enabled, and the multimodal token mask sets the visual/text gain. During inference, the language-layer updates reuse this injection context.

\subsection{Routing schedule and auxiliary objectives}
\label{app:training}
We use an exploration phase of 12,000 routing steps, during which the depth logits include the ordered historical prior, the underuse bonus, and top-$k$ logit noise. These additions are disabled after exploration and during evaluation, leaving the content scores to determine the route. To construct the historical terms, usage counts use EMA coefficient $0.99$, Dirichlet smoothing $\alpha=1$, kernel width $\gamma=1$, proximal coefficient $1/\tau=1$, and inverse-count stabilizer $10^{-3}$. The inverse-square-root bonus is normalized by its mean across depths.

For the query gate, participation and entropy losses have base weights $0.35$, and diversity has weight $1.0$. The budget and entropy terms anneal over 12,000 steps to minimum scales $0.05$ and $0.10$. The router adds query-diversity regularization and utility supervision with base coefficients $0.05$ and $0.35$, while its balancing term follows an 18,000-step schedule with minimum scale $0.25$. These auxiliary terms train the route; evaluation proceeds directly from sparse content routing to local fusion and dynamic injection.

\section{Evaluation protocols and supplementary results}
\label{app:recipe}

\subsection{Training recipe}
\paragraph{Dataset construction and benchmark de-duplication.}
\label{app:data-construction}
OMVL-16M aggregates multimodal instruction data from LLaVA-NeXT, Infinity-MM, LLaVA-OneVision, and LLaVA-Instruct. During curation, we screen candidate training records against the benchmark evaluation splits used in this paper. The screening combines image and sample identifier matching, file-hash checks, file-level matching, and image near-duplicate detection to identify exact copies and visually similar duplicates. Records matching evaluation examples are removed before the corpus is organized into the four training-stage subsets. Applying these checks across the merged source collections also identifies evaluation examples that reappear under different dataset entries or file paths. All comparisons trained on OMVL-16M use the resulting filtered subsets, keeping benchmark de-duplication consistent across architectural variants.

\paragraph{Stage-wise optimization.}
As specified in Table~\ref{tab:training}, each stage trains for one epoch on its designated OMVL-16M subset with a cosine learning-rate schedule. Stage I establishes visual--language alignment, after which stages II and III train the language and vision expert stacks in turn. Stage IV jointly tunes the system. The stack-wise comparisons use this common data and training setting, while the public-recipe study evaluates matched projector variants under a separate recipe.

\subsection{Benchmark aggregation and coverage}
\label{app:aggregation}
The main evaluation comprises MMBench, MM-Vet, TextVQA, DocVQA, POPE, MMMU, OCRBench, and ChartQA. All scores are expressed on a 0--100 scale. For model $m$, let $\mathcal{B}_m$ be its set of listed task scores and $n_m=|\mathcal{B}_m|$. Table~\ref{tab:main-vlm} reports
\begin{equation}
 \mathrm{Avg}_m=\frac{1}{n_m}\sum_{j\in\mathcal{B}_m}S_{m,j}.
\end{equation}
The superscript $(n_m)$ records each mean's coverage. With all eight scores, the means are 85.9125 for the full model, 85.225 for the projector variant without vision MoE, and 81.2375 for Qwen2.5-VL-7B, displayed as 85.9, 85.2, and 81.2. For the other rows, Qwen3-VL-8B-Instruct averages five scores to 83.3; GLM-4.6V-Flash averages three to 80.9; Step-3 averages four to 79.6; MoE-LLaVA averages six to 47.9; and EvoMoE averages three to 69.0. These partial means summarize the listed tasks. The full eight-task comparison and stack-wise analysis use $n_m=8$; emphasis in the Avg. column compares these complete rows.

LLM-only MoE scores 81.3 on average across MMBench, MM-Vet, TextVQA, DocVQA, and POPE, with MMMU 64.7, OCRBench 83.2, and ChartQA 88.1. Combining this aggregate with the three task scores gives
\begin{equation}
 \frac{5(81.3)+64.7+83.2+88.1}{8}=80.3125,
\end{equation}
which rounds to 80.3. The aggregate-derived mean is marked with $\ddagger$ in Table~\ref{tab:stack}(a). Across the one-decimal rounding interval for 81.3, the eight-task result lies within $[80.28125,80.34375)$ and rounds to the same 80.3.

Component removals and visual degradation are evaluated on MMBench, MM-Vet, TextVQA, DocVQA, and POPE. The clean averages are 81.3 for LLM-only MoE, 87.0 with the routed projector, and 87.6 with vision experts. The main benchmark comparison evaluates all eight tasks.

\subsection{Public-recipe comparison and benchmark coverage}
\label{app:external}
The public-recipe scores in Table~\ref{tab:public} follow each source's benchmark coverage: MoE-LLaVA's 61.7 averages MMBench, MM-Vet, TextVQA, and POPE, while EvoMoE's 69.0 averages MMBench, TextVQA, and POPE. Under the public training recipe, the routed projector achieves 83.7, and vision experts raise the score to 84.1. At matched parameter count, independent fusion and injection routes score 82.6. The accompanying query-gate, router, and injection replacements change one component at a time under the full projector's training and evaluation protocol.

\subsection{Sources for the external benchmark comparison}
\label{app:table2-sources}
Table~\ref{tab:main-vlm} collects results for specific model configurations from model releases and independent evaluations, marking the latter with a dagger. Each entry follows its cited protocol, and the task-count superscript specifies the coverage of its mean. Architectural comparisons use the stack-wise variants trained under the common OMVL-16M setting.

\paragraph{Qwen3-VL-8B-Instruct.}
The Qwen release~\citep{qwen3vl2025} supplies the MMBench, MMMU, and OCRBench entries. TextVQA 82.9 comes from OpenBMB's evaluation of the same checkpoint on TextVQA validation~\citep{minicpmo45card} and is marked with a dagger. OCRBench 896 is divided by ten to give 89.6.

\paragraph{GLM-4.6V-Flash and Step-3.}
The GLM-4.6V release~\citep{glm46v} reports MMMU validation 71.1 and OCRBench 84.7 for the 9B Flash model. Its comparison also reports a reproduced Step-3 OCRBench score of 83.7, marked with a dagger here. Step-3's own release supplies MMMU 74.2~\citep{step3card}. The benchmark names distinguish OCRBench from OCRBench~v2 and ChartQA from ChartQA-Pro.

\paragraph{MoE-LLaVA and EvoMoE.}
The MoE-LLaVA row uses Phi-2-2.7B with four experts, top-2 routing, and the SigLIP-384 encoder. Table~1 of \citet{yang2025stgc} reports DocVQA 25.6 and ChartQA 15.4 for this configuration, alongside the same MMBench 68.0, MM-Vet 35.9, TextVQA 57.0, and POPE 85.7 as the original report. The EvoMoE row corresponds to its OpenChat-7B configuration~\citep{jing2025evomoeexpertevolutionmixture}. Unlisted tasks retain a dash.

\subsection{Inference cost}
\label{app:profile}
Table~\ref{tab:profile-main}(b) measures peak VRAM and decoding throughput on two H800 GPUs with BF16 and $384\times384$ images. \method\ reaches 17.52 tokens/s with 52\,GB peak VRAM, compared with 19.71 tokens/s and 11\,GB for Qwen3-VL-8B. These measurements capture deployment resources alongside the model's 28B total and 9B activated parameters. Storage depends on the full expert pool, while expert execution follows the selected routes. The projector adds the work of reading ten visual streams and constructing dynamic factors for layer-wise injection.

\subsection{Expert and depth utilization}
\label{app:utilization}
Figure~\ref{fig:utilization} summarizes utilization on held-out data by normalizing expert frequencies within each MoE block and averaging the global depth distribution across samples. Language blocks distribute their load across several experts, while vision blocks show greater variation in expert usage. The projector's aggregate distribution emphasizes intermediate streams. For each individual sample, these allocations follow the sparse routes defined in Section~\ref{sec:method}.

\begin{figure}[p]
\centering
\includegraphics[width=.9\linewidth]{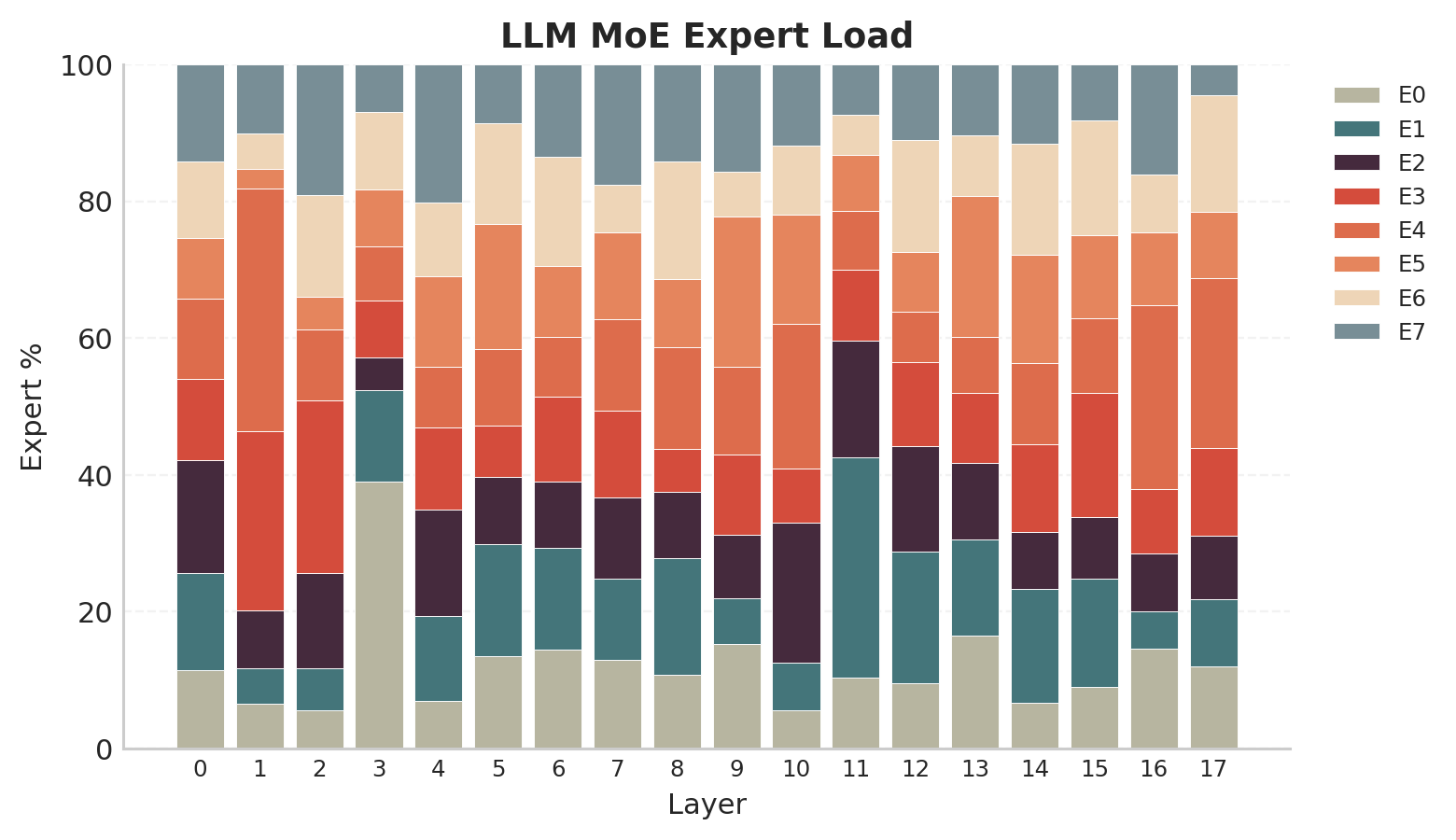}\\
\includegraphics[width=.9\linewidth]{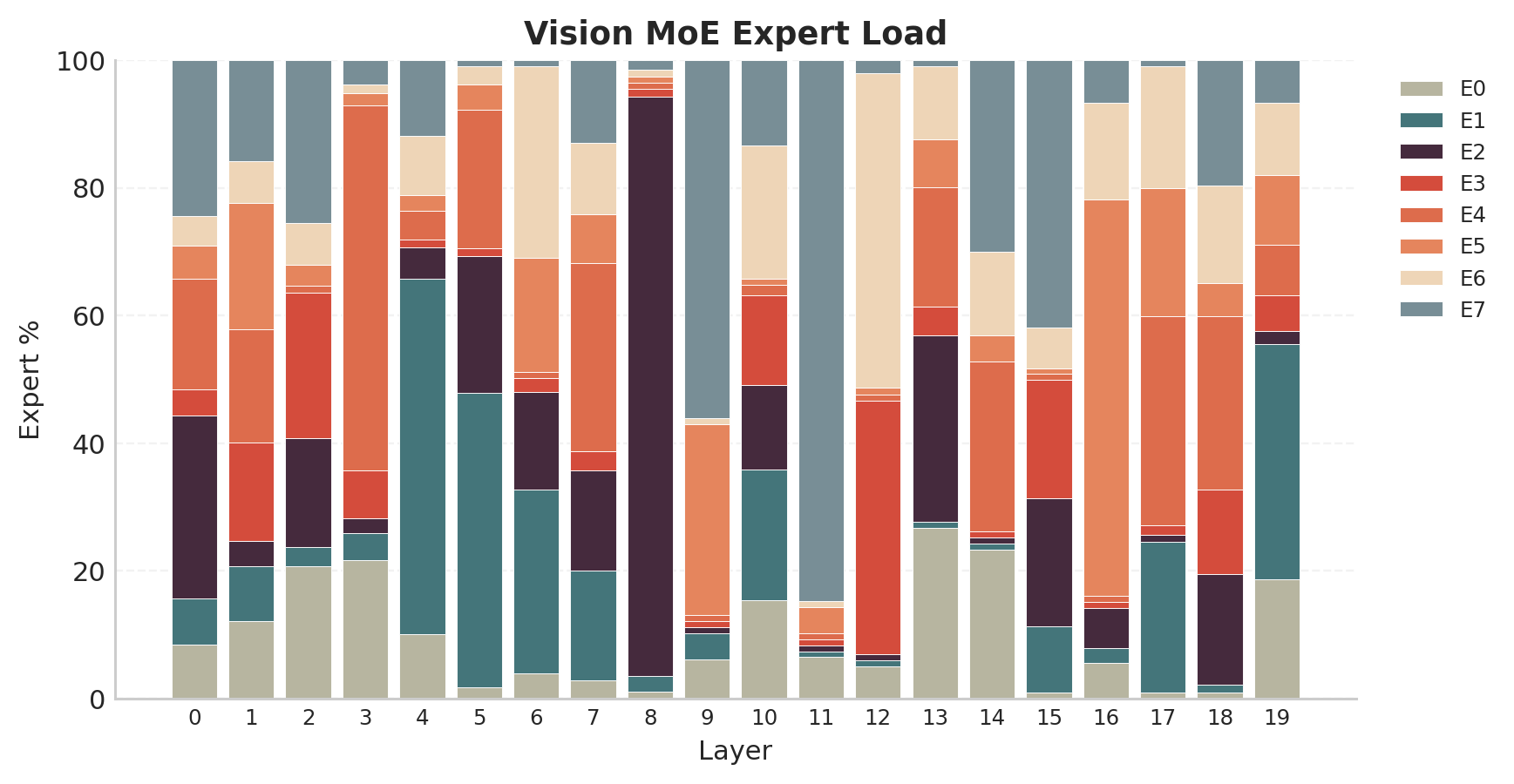}\\
\includegraphics[width=.44\linewidth]{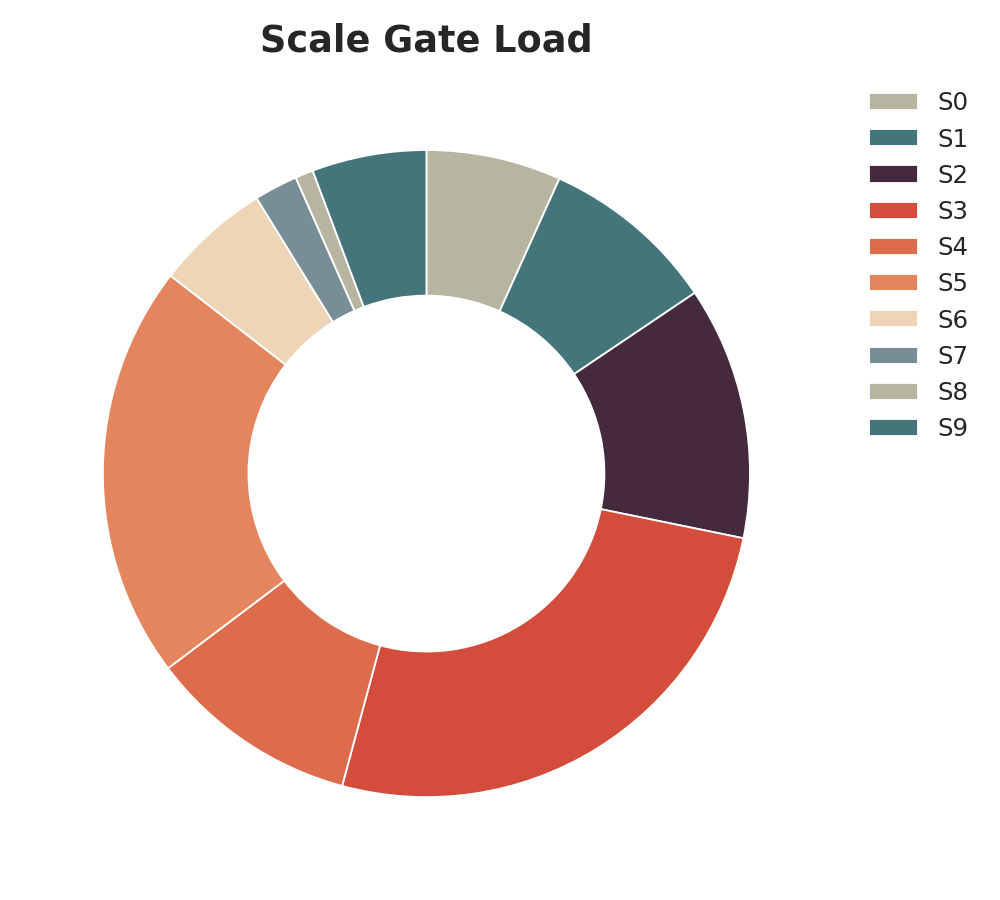}
\caption{Expert and visual-depth utilization on held-out data. Top: language-stack expert frequencies. Middle: vision-stack expert frequencies. Bottom: average global visual-depth weights. Layer axes enumerate the MoE blocks, and the depth labels S0--S9 enumerate the ten streams. E0--E7 identify parameter experts.}
\label{fig:utilization}
\end{figure}

\section{Mathematical properties of the routing operations}
\label{app:analysis}

The analysis below makes the finite-dimensional roles of query participation, ordered routing, and factorized injection explicit. It uses the same distributions and matrices as the forward computation.

\subsection{Query participation and conditioning}
For $p\in\Delta^{Q-1}$, Cauchy--Schwarz gives $\sum_i p_i^2\ge1/Q$, while $\sum_i p_i^2\le(\sum_i p_i)^2=1$. Hence $1\le\neff(p)\le Q$. A uniform distribution on $k$ entries has $\neff=k$. The floor in Equation~\ref{eq:queries} keeps every entry positive.

Let $D_x=[q_1,\ldots,q_Q]\in\R^{d\times Q}$, with nonzero singular values $\sigma_1\ge\cdots\ge\sigma_r>0$. Let $U_x$ be an orthonormal basis for the column space of $D_x$, and define $I_x(p)=D_x\diag(p)D_x^\top$.
\begin{proposition}[Conditioning on the query span]
If $p_i\ge\epsilon_q$, then
\begin{equation}
 \epsilon_q\sigma_r^2\le\lambda_{\min}(U_x^\top I_x(p)U_x)
 \le\lambda_{\max}(U_x^\top I_x(p)U_x)\le\sigma_1^2.
\end{equation}
\end{proposition}
\begin{proof}
For any $v$ in the query span,
$v^\top I_x(p)v=\sum_i p_i(q_i^\top v)^2$ lies between
$\epsilon_q\|D_x^\top v\|^2$ and $\|D_x^\top v\|^2$.
The nonzero singular values of $D_x$ bound these quadratic forms on its column space. Applying the Rayleigh quotient proves the result.
\end{proof}
The floor controls conditioning on the represented span, and the participation objective controls the concentration of its coefficients. A bank of $Q$ vectors spans at most $Q$ dimensions.

The gate is trained by differentiating through its softmax logits and participation objective. For the softmax distribution before applying the probability floor, let $J=\diag(p)-pp^\top$. A logit-gradient step induces $p^+=p-\eta J^2\nabla_pF+O(\eta^2)$. The Fisher--Rao metric on the simplex is $g_p(u,v)=\sum_i u_iv_i/p_i$ for zero-sum tangent vectors, with natural gradient $J\nabla_pF$. The first expression describes optimization through the logits; the latter describes the probability coordinates under the Fisher--Rao metric.

\subsection{Ordered prior and sparse entropy-regularized selection}
The smoothing in Equation~\ref{eq:history} assigns larger kernel weights to nearby visual depths. The resulting positive prior $\bar p$ enters an entropy-regularized score objective. For $\lambda_h,\lambda_e\ge0$ and $\lambda_h+\lambda_e>0$, consider
\begin{equation}
 \min_{v\in\Delta^{S-1}}
 -\langle v,r/T+b\rangle
 +\lambda_h\KL(v\|\bar p)
 +\lambda_e\sum_s v_s\log v_s.
 \label{eq:surrogate}
\end{equation}
\begin{proposition}[Normalized prior logits]
The unique solution of Equation~\ref{eq:surrogate} is
\begin{equation}
 v_s^*=\softmax_s\!\left(
 \frac{r_s/T+b_s+\lambda_h\log\bar p_s}{\lambda_h+\lambda_e}\right).
\end{equation}
With a support budget of at most $k$, the optimal support contains the $k$ largest effective logits, with ties permitting multiple optimal supports.
\end{proposition}
\begin{proof}
Collecting the entropy terms gives a strictly convex objective with coefficient $\lambda_h+\lambda_e$. A Lagrange multiplier for $\sum_s v_s=1$ gives the softmax expression. On any fixed support $J$, the optimal value differs from a constant by
$-(\lambda_h+\lambda_e)\log\sum_{s\in J}\exp(z_s/(\lambda_h+\lambda_e))$,
where $z_s=r_s/T+b_s+\lambda_h\log\bar p_s$. The largest $k$ logits maximize this partition sum.
\end{proof}
The denominator sets the objective's effective temperature. The depth kernel supplies the ordered log prior, which combines with content scores and exploration bonuses in the routing logits. Training schedules these terms before sparse softmax selects the active depths.

\subsection{Rank, norms, and propagation of injected residuals}
The factorization in Equation~\ref{eq:dynamic} gives $\rank(T_s)\le r_T$. Composition with $G_\ell$ gives $\rank(G_\ell T_s)\le\min(r_T,r_G)$. With one stream per anchor, the pre-gain transformation at each injection layer has rank at most 32. The relevant matrix norm relationships are
\begin{equation}
 \|A\|_{\mathrm{op}}\le\|A\|_F\le\|A\|_*
 \le\sqrt{\rank(A)}\,\|A\|_F.
\end{equation}
For finite matrices, the Hilbert--Schmidt norm is the Frobenius norm. The shared factor regularizer uses
\begin{equation}
 \|UV^\top\|_*\le\tfrac12(\|U\|_F^2+\|V\|_F^2).
\end{equation}
This bounds the nuclear norm of the represented adapter through its learned factors.

Consider a reference recurrence $H_{\ell+1}=\Phi_\ell(H_\ell)$ and an injected recurrence $\widetilde H_{\ell+1}=\Phi_\ell(\widetilde H_\ell+\Delta_\ell)$. If $\Phi_\ell$ is $L_\ell$-Lipschitz on the trajectories and $\|\Delta_\ell\|_F\le d_\ell$, then
\begin{equation}
 \|\widetilde H_{L+1}-H_{L+1}\|_F
 \le \sum_{\ell=1}^{L}d_\ell\prod_{j=\ell}^{L}L_j
\end{equation}
for equal starting states. Applying the Lipschitz inequality at each layer and unrolling the recurrence gives the bound, with the product beginning at the current injection layer. Each residual's contribution is therefore weighted by its amplification through the remaining language computation.

\subsection{A minimal example of visual access}
Let $V_1,V_2$ be independent fair bits and let the prompt index $J\in\{1,2\}$ specify the target $Y=V_J$. An interface exposing only $V_1$ achieves error $1/2$ when $J=2$, even with an unrestricted downstream predictor. Exposing the prompt-selected bit $V_J$ gives zero error, as does preserving both bits. The example separates visual access from downstream prediction: answering the question requires the relevant variable to cross the interface. In \method, the routed projector implements this access decision over intermediate visual streams, and the architectural comparisons evaluate its contribution to the model.

\end{document}